\documentclass[11pt]{article}

\usepackage{times}
\usepackage{epsfig}
\usepackage{graphicx}
\usepackage{amsmath}
\usepackage{amssymb}
\usepackage{amsfonts}
\usepackage{bbm}
\usepackage{eucal}
\usepackage{slashbox}
\usepackage{multirow}

\usepackage{sidecap}
\usepackage{ltablex} 
\usepackage{array}
\usepackage{booktabs}

\usepackage{algorithmic}
\usepackage{algorithm}
\usepackage{url}

\usepackage{color,soul}
\usepackage{caption}
\usepackage{blindtext}

\newcommand{\argmax}{\operatornamewithlimits{arg\,max}}

\newcommand{\argmin}{\operatornamewithlimits{arg\,min}}

\DeclareMathAlphabet{\mathcal}{OMS}{cmsy}{b}{n}

\newtheorem{theorem}{Theorem}[section]
\newtheorem{lemma}[theorem]{Lemma}

\newenvironment{proof}[1][Proof]{\begin{trivlist}
\item[\hskip \labelsep {\bfseries #1}]}{\end{trivlist}}

\newcommand{\qed}{\nobreak \ifvmode \relax \else
      \ifdim\lastskip<1.5em \hskip-\lastskip
      \hskip1.5em plus0em minus0.5em \fi \nobreak
      \vrule height0.75em width0.5em depth0.25em\fi}

\title{Latent Fisher Discriminant Analysis}

\author{
Gang Chen 
\\
Department of Computer Science and Engineering \\
 SUNY at Buffalo\\
\texttt{gangchen@buffalo.edu} \\
}

\begin{document}

\maketitle

\begin{abstract}

Linear Discriminant Analysis (LDA) is a well-known method for dimensionality reduction and classification. Previous studies have also extended the binary-class case into multi-classes. However, many applications, such as object detection and keyframe extraction cannot provide consistent instance-label pairs, while LDA requires labels on instance level for training. Thus it cannot be directly applied for semi-supervised classification problem. In this paper, we overcome this limitation and 
propose a latent variable Fisher discriminant analysis model. We relax the instance-level labeling into bag-level, is a kind of semi-supervised (video-level labels of event type are required for semantic frame extraction) and  
incorporates a data-driven prior over the latent variables.  Hence, our method combines the latent variable inference and dimension reduction in an unified bayesian framework. We test  our method on MUSK and Corel data sets and yield competitive results compared to the baseline approach. We also demonstrate its capacity on the challenging TRECVID MED11 dataset for semantic keyframe extraction and conduct a human-factors ranking-based experimental evaluation, which clearly demonstrates our proposed method consistently extracts more semantically meaningful keyframes than challenging baselines.

\end{abstract}

\section{Introduction}

Linear Discriminant Analysis (LDA) is a powerful tool for dimensionality reduction and classification that projects high-dimensional data into a low-dimensional space where the data achieves maximum class separability \cite{Duda00,Fukunaga90}. The basic idea in classical LDA, known as the Fisher Linear Discriminant Analysis (FDA) is to obtain the projection matrix by minimizing the within-class distance and maximizing the between-class distance simultaneously to yield the maximum class discrimination. It has been proved analytically that the optimal transformation is readily computed by solving a generalized eigenvalue problem \cite{Fukunaga90}. In order to deal with multi-class scenarios \cite{Duda00}, LDA can be easily extended from binary-case and generally used to find a subspace with $d-1$ dimensions for multi-class problems, where $d$ is the number of classes in the training dataset. Because of its effectiveness and computational efficiency, it has been applied successfully in many applications, such as face recognition \cite{Belhumeur97} and microarray gene expression data analysis. Moreover, LDA was shown to compare favorably with other supervised dimensionality reduction methods through experiments \cite{Sugiyama10}.

However, LDA expects instance/label pairs which are surprisingly prohibitive especially for large training data. In the last decades, semi-supervised methods have been proposed to utilize unlabeled data to aid classification or regression tasks under situations with limited labeled data, such as Transductive SVM (TSVM) \cite{Vapnik98,Joachims99} and Co-Training \cite{Blum98}.
Correspondently, it is reasonable to extend the supervised LDA into a semi-supervised method, and many approaches \cite{Cai07,Zhang08,Sugiyama10} have been proposed. Most of these methods are based on transductive learning. In other words, they still need instance/label pairs. However, many real applications require bag-level labeling \cite{Andrews02}, such as object detection \cite{FeGiMcPAMI2010} and event detection \cite{2010trecvidover}.

In this paper, we propose a Latent Fisher Discriminant Analysis model (or LFDA in short), which generalizes Fisher LDA model \cite{Duda00}. Our model is inspired by MI-SVM \cite{Andrews02} or latent SVM \cite{FeGiMcPAMI2010} and multiple instance learning problems \cite{Dietterich97,Maron98}. On the one hand, recently applications in image and video analysis require a kind of bag-level label. Moreover, using latent variable model for this kind of problem shows great improvement on object detection \cite{FeGiMcPAMI2010}. On the other hand, the requirement of instance/label pairs in the training data is surprisingly prohibitive especially for large training data. The bag-level labeling methods are a good solution to this problem. 

MI-SVM or latent SVM is a kind of discriminative model by maximizing a posterior probability. 
Our model unify the discriminative nature of the Fisher linear discriminant with a data driven Gaussian mixture prior over the training data in the Bayesian framework. By combining these two terms in one model, we infer latent variables and projection matrix in alternative way until convergence. We demonstrate this capability on MUSK and Corel data sets for classification, and on TRECVID MED11 dataset for keyframe extraction on five video events \cite{2010trecvidover}.  


\section{Related Work}
\label{sec:related}
Linear Discriminant Analysis (LDA) has been a popular method for dimension reduction and classification. It searches a projection matrix that simultaneously maximizes the between-class dissimilarity and minimizes the within-class dissimilarity to increase class separability, typically for classification applications. LDA has attracted an increasing amount of attention in many applications because of its effectiveness and computational efficiency. Belhumeur et al proposed PCA+LDA \cite{Belhumeur97} for face recognition. Chen et al projects the data to the null space of the within-class scatter matrix and then maximizes the between-class scatter in this space \cite{Chen00} to deal with the situation when the size of training data is smaller than the dimensionality of feature space. \cite{Wang04} combines the ideas above, maximizes the between-class scatter matrix in the range space and the null space of the within-class scatter matrix separately and then integrates the two parts together to get the final transformation.  \cite{Ye05} is also a two-stage method which can be divided into two steps: first project the data to the range space of the between-class scatter matrix and then apply LDA in this space. To deal with non-linear scenarios, the kernel approach \cite{Vapnik98} can be applied easily via the so-called kernel trick to extend LDA to its kernel version, called kernel discriminant analysis \cite{Baudat00}, that can project the data points nonlinearly. Recently, sparsity induced LDA is also proposed \cite{Francisco12}.

However, many real-world applications only provide labels on bag-level, such as object detection and event detection. LDA, as a classical supervised learning method, requires a training dataset consisting of instance and label pairs, to construct a classifier that can predict outputs/labels for novel inputs. However, directly casting LDA as a semi-supervised method is challenging for multi-class problems. Thus, in the last decades, semi-supervised methods become a a hot topic. 
One of the main trend is to extend the supervised LDA into a semi-supervised method \cite{Cai07,Zhang08,Sugiyama10}, which attempts to utilize unlabeled data to aid classification or regression tasks under situations with limited labeled data. \cite{Cai07} propose a novel method, called Semi-supervised Discriminant Analysis, which makes use of both labeled and unlabeled samples. The labeled data points are used to maximize the separability between different classes and the unlabeled data points are used to estimate the intrinsic geometric structure of the data. \cite{Sugiyama10} propose a semi-supervised dimensionality reduction method which preserves the global structure of unlabeled samples in addition to separating labeled samples in different classes from each other. 
Most of these semi-supervised methods model the geometric relationships between all data points in the form of a graph and then propagate the label information from the labeled data points through the graph to the unlabeled data points. Another trend prefers to extent LDA into an unsupervised senarios. For example, Ding et al propose to combine LDA and K-means clustering into the LDA-Km algorithm \cite{Ding07} for adaptive dimension reduction. In this algorithm, K-means clustering is used to generate class labels and LDA is utilized to perform subspace selection.

Our solution is a new latent variable model called Latent Fisher Discriminant Analysis (LFDA), which complements 
existing latent variable models that have been popular in the recent vision literature \cite{FeGiMcPAMI2010} by making it possible to include the latent variables into Fisher discriminant analysis model. Unlike previous latent SVM \cite{FeGiMcPAMI2010} or MI-SVM \cite{Andrews02} model, we extend it with prior data distribution to maximize a joint probability when inferring latent variables. Hence, our method combines the latent variable inference and dimension reduction in an unified Bayesian framework. 
\section{Latent Fisher discriminant analysis}
\label{sec:lfda}


We propose a LFDA model by including latent variables into the Fisher 
discriminant analysis model. 
Let ${\mathcal{X}}=\{{\bf x_{1}}, {\bf x_{2}},...,{\bf x_{n}}\}$  
represent $n$ bags, and the corresponding labels ${\mathcal{L}}=\{l_{1}, l_{2},...,l_{n} \}$. For each ${\bf 
x_{i}} \in \mathcal{X}$, ${\bf x_{i}}$ can be treated as a bag (or video)
in \cite{Andrews02}, and its label $l_{i}$ is categorical and assumes 
values in a finite set, e.g. $\{1,2,..., C\}$.  Let ${\bf x_{i}} \in 
\mathbb{R}^{d \times n_{i}}$, which means it contains $n_{i}$ 
instances (or frames), ${\bf x_{i}} = 
\{x_{i}^{1},x_{i}^2,...,x_{i}^{n_{i}}\}$, and its $j^{th}$ instance is 
a vector in $\mathbb{R}^{d}$, namely $x_{i}^j\in \mathbb{R}^{d}$. Fisher's linear 
discriminant analysis pursue a subspace $\mathcal{Y}$ to 
separate two or more classes. In other words, for any instance $x \in 
\mathcal{X}$, it searches for a projection $f: x\rightarrow y$, where 
$y \in  \mathbb{R}^{d'}$ and $d' \le d$. In general, $d'$ is decided 
by $C$, namely $d' = C -1$. Suppose the projection matrix is 
$\mathcal{P}$, and $y = f(x) =\mathcal{P}  x$, then latent Fisher LDA 
proposes to minimize the following ratio:
%
%
\begin{align}\label{eq:Eq1}
\centering
({\mathcal{P}^*}) = \argmin_{\mathcal{P},z}J({\mathcal{P}, z}) 
=\argmin_{\mathcal{P},z} 
\textrm{trace}  \left(  \frac{\mathcal{P}^{T} 
\Sigma_{w}(x,z) \mathcal{P}}{\mathcal{P}^{T}\Sigma_{b} (x,z) \mathcal{P}}  + \beta \mathcal{P}^{T}\mathcal{P} \right)&
\end{align}
where $z$ is the latent variable, $\beta$ is an weighting parameter for regularization term. The set $z\in Z(x)$ defines the possible latent values for a sample $x$. In our case, $z \in \{1,2,..., C\}$. $\Sigma_{b}(x,z)$ is between 
class scatter matrix and $\Sigma_{w}(x,z)$ is within class scatter 
matrix. However, LDA is dependent on a categorical variable (i.e. the 
class label) for each instance $x$ to compute $\Sigma_{b}$ and $\Sigma_{w}$.  In our case, we only know bag-level labels, not on instance-level labels.  
To minimize $J({\mathcal{P}}, z)$, we need to solve $z(x)$ for any 
given $x$. This problem is a chicken and egg problem, and can be 
solved by alternating algorithms, such as EM \cite{Dempster77}. In 
other words, solve $\mathcal{P}$ in Eq. (\ref{eq:Eq1}) with fixed $z$, and vice versa in an 
alternating strategy. 

\subsection{Updating $z$}
Suppose we have found the projection matrix $\mathcal{P}$, and 
corresponding subspace $\mathcal{Y}=\mathcal{P}\mathcal{X}$, where 
$\mathcal{Y}= \{{\bf y_{1}}, {\bf y_{2}},...,{\bf y_{n}}\}$. Instead of inferring latent variables at instance-level in latent SVM, we propose latent variable inference at clustering-level in the projected space $\mathcal{Y}$. That means all elements in the same cluster have same label. Such assumption is reasonable because elements in the same cluster are close to each other. On the other hand, cluster-level inference can speed up the learning process. 
We extend mixture discriminative analysis model in \cite{Hastie96} by incorporating latent variables over all instances for an given class. As in \cite{Hastie96}, we assume each class $i$ is a $K$ components of Gaussians, 
\begin{equation}\label{eq:Eq4}
\centering
p(x|\lambda_{i})=\sum_{j=1}^K\pi_{i}^{j}g(x| \mu_{i}^{j}, \Sigma_{i}^{j})
\end{equation}
where $x$ is a $d$-dimensional continuous-valued data vector (i.e. measurement or features); $\pi_{i} = \{\pi_{i}^{j}\}_{j = 1}^{K}$ are the mixture weights, and $g(x| \mu_{i}^{j}, \Sigma_{i}^{j})$, $j \in [1,K]$, are the component Gaussian densities with $\mu_{i} = \{\mu_{i}^{j} \}_{j=1}^K$ as mean and $\Sigma_{i} = \{\Sigma_{i}^{j}\}_{j=1}^K$ as covariance. $\lambda_{i}$ is the parameters for class $i$ which we need to estimate, $\lambda_{i}=\{\pi_{i}, \mu_{i}, \Sigma_{i}\}$.

Hence, for each class $i \in \{1,2,\dots,C\}$, we can estimate $\lambda_{i}$ and get its $K$ subsets $S_{i}= \{S_{i}^1,S_{i}^2,...,S_{i}^K\}$ with EM algorithm. 
Suppose we have the discriminative weights (or posterior probability) for the $K$ centers in each class, $w_{i} =\{w_{i}^1,w_{i}^2,...,w_{i}^K\}$, which are the posterior probability determined by the latent FDA and will be discussed later. We maximize one of the following two equations:

\begin{subequations}\label{eq:Eq3}
Maximizing a posterior probability:
\begin{align}\label{a}
\centering
        \mu_{i}^j = \argmax_{\mu_{i}^{j} \in \mu_{i}, j \in [1,K]} w_{i}  = \argmax_{\mu_{i}^{j} \in \mu_{i}, j \in[1,K]}  p(z_{i} | \mu_{i}, \mathcal{P})
\end{align}

Maximizing the joint probability with prior:
\begin{align}\label{b}
\centering
\mu_{i}^j  = \argmax_{\mu_{i}^{j} \in \mu_{i}, j \in[1,K]} (\pi_{i} \circ w_{i})
\end{align}
\end{subequations}
where $z_{i}$ is the latent label assignment, $\pi_{i}$ is the prior clustering distributions for $\lambda_{i}$ in class $i$, $w_{i}$ is the posterior (or weight) determined by kNN voting (see further) in the subspace and $\circ$ is the pointwise production or Hadamard product. 
We treat Eq. (\ref{a}) as the latent Fisher discriminant analysis model (LFDA), because it takes the same strategy as the latent SVM model \cite{Andrews02,FeGiMcPAMI2010}. As for Eq. (\ref{b}), we extend LFDA by combining the both factors (representative and discriminative) together, and find the cluster $S_{i}^{j}$ in class $i$ by maximizing Eq. (\ref{b}). In a sense, Eq. (\ref{b}) considers the prior distribution from the training dataset,
thus, we treat it as the joint latent Fisher discriminant analysis model (JLFDA) or LFDA with prior. In the nutshell, we propose a way to formulate discriminative and generative methods under the unified Bayesian framework. 
We comparatively analyze both of these models (Section \ref{sec:experiments}).

Consequently, if we select the cluster $S_{i}^j$ with the mean 
$\mu_{i}^j$ which maximizes the above equation for class $i$, we can relabel all 
samples $x$ positive for class $i$ and the rest negative, subject to  $y={\mathcal{P}}x$ and $y \in S_{i}^{j}$. 
Then, we construct a new training data ${\mathcal{X^+}}=\{{\bf x_{1}^+}, {\bf x_{2}^+},...,{\bf x_{n}^+}\}$, with labels ${\mathcal{L^+}}=\{{\bf z_{1}^+}, {\bf z_{2}^+},...,{\bf z_{n}^+}\}$, where ${\bf x_{i}^+}= S_{i}^{j}$ for class $i$ with $n_{i}^{j'}$ elements, and its labels ${\bf z_{i}^+} = \{z_{i}^1, z_{i}^2,...,z_{i}^{n_{i}^{j'}}\} $ on instance level. Obviously, ${\bf x_{i}^+} \subseteq {\bf x_{i}}$ and ${\mathcal{X^+}}$ is a subset of ${\mathcal{X}}$. The difference between ${\mathcal{X^+}}$ and ${\mathcal{X}}$ lies that every element $x_{i}^+\in{\bf x_{i}^+}$ has label $z(x_{i}^+)$ decided by Eq. (\ref{eq:Eq3sub}), while ${\bf x_{i}} \subset \mathcal{X}$ only has bag level label.

\subsection{Updating projection $\mathcal{P}$}

When we have labels for the new training data ${\mathcal{X^+}}$, we use the Fisher LDA to minimize $J({\mathcal{P}}, z)$. Note that Eq. (\ref{eq:Eq1}) is invariant to the scale of the vector $\mathcal{P}$. Hence, we can always choose $\mathcal{P}$ such that the denominator is simply  $\mathcal{P}^{T} \Sigma_{b} \mathcal{P} = {\bf 1}$. For this reason we can transform the problem of minimizing Eq. (\ref{eq:Eq1}) into the following constrained optimization problem
\cite{Duda00,Fukunaga90,Ye07}:
\begin{align}\label{eq:Eq4}
\centering
 {\mathcal P^*} & =\argmin_{{\mathcal P}}    \textrm{trace}  \left( \mathcal{P}^{T}\Sigma_{w} (x,z) \mathcal{P}  + \beta \mathcal{P}^{T}\mathcal{P} \right) & \nonumber \\
& \quad \textrm{s. t. }  {\mathcal{P}^{T} \Sigma_{b}(x,z) \mathcal{P}} ={\bf 1} &
\end{align}
where {\bf 1} is the identity matrix in $\mathbb{R}^{d'\times d'}$. 
The optimal Multi-class LDA consists of the top eigenvectors of $(\Sigma_{w}(x,z) + \beta)^{\dagger}\Sigma_{b} (x,z)$ corresponding to the nonzero eigenvalues \cite{Fukunaga90}, here $(\Sigma_{w}(x,z)+\beta)^{\dagger}$ denotes the pseudo-inverse of $\Sigma_{w}(x,z) + \beta$. After we calculated $\mathcal{P}$, we can project ${\mathcal{X^+}}$ into subspace ${\mathcal{Y^+}}$. Note that in the subspace ${\mathcal{Y^+}}$, any $y^+ \in{\mathcal{Y^+}}$ preserves the same labels as in the original space. In other words, ${\mathcal{Y^+}}$ has corresponding labels ${\mathcal{L^+}}$ at element level, namely $z(y^+) = z(x^+)$.

In general, multi-class LDA \cite{Ye07} uses kNN to classify new input 
data. We compute $w_{i}$ using the following kNN strategy: for each 
sample $x \in {\mathcal{X}}$, we get $y = \mathcal{P}x$ by projecting it into subspace ${\mathcal{Y}}$. Then, for $y\in {\mathcal{Y}}$, we choose its $N$ nearest neighbors from ${\mathcal{Y^+}}$, and use their labels to voting each cluster $S_{i}^j$ in each class $i$. Then, we compute the following posterior probability: 
\begin{align}\label{eq:Eq5}
\centering
w_{i}^{j} & = p(z_{i}=1| \mu_{i}^j) = p(\mu_{i}^j | z_{i}=1) p(z_{i}=1) & \nonumber \\
& =  p(z_{i}=1) \frac{p(\mu_{i}^j, z_{i}=1)}{\sum_{i=1}^C p(\mu_{i}^j, z_{i}=1)} &
\end{align}
It counts all $y \in {\mathcal{Y}}$ fall into $N$ nearest neighbor of $\mu_{i}^{j}$ with label $z_{i}$. Note that
kNN is widely used as the classifier in the subspace after LDA 
transformation. Thus, Eq. (\ref{eq:Eq5}) consider all training data to 
vote the weight for each discriminative cluster $S_{i}^j$ in every 
class $i$. Hence, we can find the most discriminative cluster $S_{i}^{j}$, s.t. $w_{i}^{j} > w_{i}^k$, $k \in [1, K], k \neq j$. 

\textbf{Algorithm.}
We summarize the above discussion in pseudo code. To put simply, we update ${\mathcal{P}}$ and $z$ in an alternative manner, and accept the new projection matrix ${\mathcal{P}}$ with LDA on the relabeled instances. Such algorithm can always convenge in around 10 iterations. After we learned matrix $\mathcal{P}$ and $\{\lambda_{i}\}_{i=1}^{C}$ by maximizing Eq. ({\ref{eq:Eq3}), we can use them to select representative and discriminative frames from video datasets by nearest neighbor searching.
\begin{algorithm}[h!]
  \footnotesize
\caption{} 
\label{alg1}
\textbf{Input}: training data ${\mathcal{X}}$ and its labels $\mathcal{L}$ at video level, $\beta$, $K$, $N$, $T$ and $\epsilon$.\\
\textbf{Output}: $\mathcal{P}$, $\{\lambda_{i}\}_{i=1}^{C}$
\begin{algorithmic}[1]
\STATE Initialize $\mathcal{P}$ and $w_{i}$;
\FOR{$Iter=1$ to $T$}
\FOR{$i=1;  i<=C;  i++$}
\STATE Project all the training data ${\mathcal{X}}$ into subspace ${\mathcal{Y}}$ using ${\mathcal{Y}}$=$\mathcal{P}{\mathcal{X}}$;
\STATE For each class $i \in [1,C]$, using Gaussian mixture model to partition its elements in the subspace, and compute $\lambda_{i}$ = $\{\pi_{i}, \mu_{i}, \Sigma_{i}\}$;
\STATE Maximize Eq. (\ref{eq:Eq3}) to find $S_{i}^j$ with center $\mu_{i}^j$;
\STATE Relabel all elements positive in the cluster $S_{i}^j$ for class $i$ according to Eq. (\ref{eq:Eq3sub});
\ENDFOR
\STATE Update $z$ and construct the new subset ${\mathcal{X^+}}$ and its labels $\mathcal{L}^+$ for all $C$ classes;
\STATE Do Fisher linear discriminant analysis and update $\mathcal{P}$
\STATE if $\mathcal{P}$ converge (change less than $\epsilon$), then break 
\STATE Compute $N$ nearest neighbors for each training data, and calculate discriminative weight $w_{i}$ for each class $i$ according to Eq. (\ref{eq:Eq5}).
\ENDFOR
\STATE Return ${\mathcal{P}}$ and cluster centers $\{\lambda_{i} \}_{i=1}^C$ learned respectively for all $C$ classes;
\end{algorithmic}
\end{algorithm}

\subsection{Convergence analysis}
Our method updates latent variable $z$ and then $\mathcal{P}$ in an alternative manner. Such strategy can be attributed to the hard assignment of EM algorithm. Recall the EM approach: 
\begin{align}\label{eq:Eq7}
\centering
\mathcal{P^*} = \argmax_{\mathcal{P}} p(\mathcal{X}, \mathcal{L} | \mathcal{P})
= \argmax_{\mathcal{P}} \sum_{i=1}^C p(\mathcal{X}, \mathcal{L} | z_{i}) p(z_{i}| \mathcal{P})
\end{align}
then the likelihood can be optimized using iterative use of the EM algorithm.
\begin{theorem}
Assume the latent variable $z$ is inferred for each instance in $\mathcal{X}$, then to maximize the above function is equal to maximize the following auxiliary function 
\begin{align}\label{eq:Eq8}
\centering
\mathcal{P} = \argmax_{\mathcal{P}} \sum_{i=1}^C p(z_{i} | \mathcal{X}, \mathcal{L}, \mathcal{P}^\prime) ln \bigg(p(\mathcal{X}, \mathcal{L} | z_{i}) p(z_{i} | \mathcal{P})\bigg)
\end{align}
\end{theorem}
This proof can be shown using Jensen's inequality.
\begin{lemma} The hard assignment of latent variable $z$ by maximizing Eq. (\ref{eq:Eq3}) is a special case of EM algorithm.
\end{lemma}
\begin{proof}
\begin{align}\label{eq:Eq9}
\centering
& \sum_{i=1}^C p(z_{i} | \mathcal{X}, \mathcal{L}, \mathcal{P}^\prime) ln \bigg(p(\mathcal{X}, \mathcal{L} | z_{i}) p(z_{i} | \mathcal{P})\bigg) \nonumber \\
=& \sum_{i=1}^C p(z_{i} | \mathcal{X}, \mathcal{L}, \mathcal{P}^\prime) ln (p(z_{i} |\mathcal{X}, \mathcal{L}, P)) + \sum_{i=1}^C p(z_{i} | \mathcal{X}, \mathcal{L}, \mathcal{P}^\prime) ln (p(\mathcal{X}, \mathcal{L} | \mathcal{P})) \nonumber \\
= & \sum_{i=1}^C p(z_{i} | \mathcal{X}, \mathcal{L}, \mathcal{P}^\prime) ln (p(z_{i} |\mathcal{X}, \mathcal{L}, \mathcal{P})) +  ln (p(\mathcal{X}, \mathcal{L} | \mathcal{P}))
\end{align}
Given $\mathcal{P}^\prime$, we can infer the latent variable $z$. Because the hard assignment of $z$, the first term in the right hand side of Eq. (\ref{eq:Eq9}) assigns $z_{i}$ into one class. Note that $p(z | \mathcal{X}, \mathcal{L}, \mathcal{P}) ln (p(z |\mathcal{X}, \mathcal{L}, \mathcal{P}))$ is a monotonically increasing function, which means that by maximizing the posterior likelihood $p(z | \mathcal{X}, \mathcal{L}, \mathcal{P})$ for each instance, we can maximize Eq. (\ref{eq:Eq9}) for the hard assignment case in Eq. (\ref{eq:Eq3}). Thus, the updating strategy in our algorithm is a special case of EM algorithm, and it can converge into a local maximum as EM algorithm. Note that in our implement, we infer the latent variable in cluster level. In other words, to maximize $p(z_{i} | \mathcal{X}, \mathcal{L}, \mathcal{P}^\prime)$, we can include another latent variable $\pi_{j}, j\in[1,K]$. In other words, we need to maximize $\sum_{j=1}^K p(z_{i}, \pi_{j} | \mathcal{X}, \mathcal{L}, \mathcal{P}^\prime)$, which we can recursively determine the latent variable $\pi_{i}$ using an embedded EM algorithm. Hence, our algorithm use two steps of EM algorithm, and it can converge to a local maximum. Refer \cite{Wu83} for more details about the convergence of EM algorithm.
\end{proof}

\subsection{Probabilistic understanding for the model}
The latent SVM model \cite{FeGiMcPAMI2010,Andrews02} propose to label instance 
$x_{i}$ in positive bag, by maximize $p(z(x_{i}) =1| x_{i})$, which is the optimal 
Bayes decision rule. 
\begin{SCfigure}[\sidecaptionrelwidth][h!]
\centering
\begin{minipage}{0.5\linewidth}
\includegraphics[trim = 20mm 90mm 8mm 60mm, clip=true, width=15.0cm]{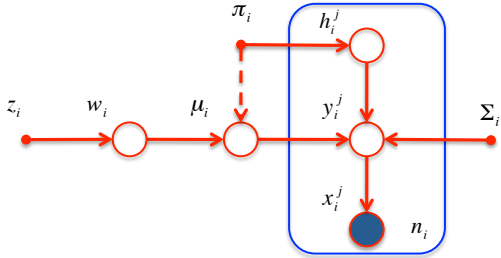}
\end{minipage}
\caption{Example of graphical representation only for one class (event). $h_{i}^j$ is the hidden variable, $x_{i}^j$ is the observable input, $y_{i}^j$ is the projection of $x_{i}^j$ in the subspace, $j \in[1,n_{i}]$, and $n_{i}$ is the number of total training data for class $i$. The $K$ cluster centers $\mu_{i} = \{\mu_{i}^1, \mu_{i}^2,...,\mu_{i}^{K}\}$ is determined by both $\pi_{i}$ and $w_{i}$. The graphical model of our method is similar to GMM model in vertical. By adding $z_{i}$ into LDA, the graphical model can handle latent variables.}
\label{fig:long}
\end{SCfigure}
Similarly, Eq. (\ref{a}) takes the same strategy as latent SVM to maximize a 
posterior probability. Moreover, instead of only maximizing the 
$p(z=1|x)$, we also maximize the joint probability $p(z =1, 
x)$, using the Bayes rule, $p(z =1, x) = p(x) p(z=1|x)$. In this 
paper, we use Gaussian mixture model to approximate the prior $p(x)$ in this generative model. 
We argue that to maximize a joint probability is reasonable, because it 
considers both discriminative (posterior probability) and 
representative (prior) property in the video dataset. We give the 
graphical representation of our model in Fig. (\ref{fig:long}). 

\section{Experiments and results}
\label{sec:experiments}
In this section, we perform experiments on various data sets to evaluate the proposed techniques and compare it to other baseline methods. For all the experiments, we set $T$ = 20 and $\beta=40$; and initialize uniformly weighted 
$w_{i}$ and projection matrix $\mathcal{P}$ with LDA. 
\subsection{Classification on toy data sets}
The MUSK data sets{\footnote{www.cs.columbia.edu/~andrews/mil/datasets.html}} are the benchmark data sets used in virtually all previous approaches and have been described in detail in the landmark paper \cite{Dietterich97}. Both data sets, MUSK1 and MUSK2, consist of descriptions of molecules using multiple low-energy conformations. Each conformation is represented by a 166-dimensional feature vector derived from surface properties. MUSK1 contains on average approximately 6 conformation per molecule, while MUSK2 has on average more than 60 conformations in each bag. 
The Corel data set consists three different categories (ÒelephantÓ, ÒfoxÓ, ÒtigerÓ) , and each instance is represented with 230 dimension features, characterized by color, texture and shape descriptors. The data sets have 100 positive and 100 negative example images. The latter have been randomly drawn from a pool of photos of other animals. We first use PCA reduce its dimension into 40 for our method. For parameter setting, we set $K$=3, $T$ = 20 and $N=4$ (namely the 4-Nearest-Neighbor (4NN) algorithm is applied for classification). The averaged results of ten 10-fold cross-validation runs are summarized in Table (\ref{tab:tab1}). We set LDA{\footnote{we use the bag label as the instance label to test its performance}} and MI-SVM as our baseline. We can observe that both LFDA and JLFDA outperform MI-SVM on MUSK1 and Fox data sets, while has comparative performance as MI-SVM on the others. 
\begin {table}[tp!] 
\begin{center}
    \begin{tabular}{| c | c | c | c | c | c |}
    \hline
     Data set &inst/Dim & MI-SVM & LDA & LFDA & JLFDA \\ \hline
    MUSK1 & 476/166& 77.9  & 70.4 & 81.4 &  87.1 \\ \hline
    MUSK2 & 6598/166& 84.3 & 51.8 & 76.4 &  81.3 \\ \hline
    \hline
    Elephant & 1391/230 &81.4   &  70.5 & 74.5 &  79.0  \\ \hline
Fox & 1320/230 &57.8 & 53.5 & 61.5 & 59.5  \\ \hline
Tiger & 1220/230 & 84.0 & 71.5 & 74.0 & 80.5  \\ \hline

    \end{tabular}
\end{center}
    \caption{Accuracy results for various methods on MUSK and Corel data sets. Our approach outperform LDA on both datasets, and we get better result than MI-SVM on MUSK1 and Fox data set.}
    \label{tab:tab1}
\end {table}
\subsection{Semantic keyframe extraction}
We conduct experiments on the challenging TRECVID MED11 
dataset{\footnote{http://www.nist.gov/itl/iad/mig/med11.cfm}}. It 
contains five events: attempting a board trick
feeding an animal, landing a fish, wedding ceremony and working on a 
woodworking project. All of five events consist of a number of human 
actions, processes, and activities interacting with other people 
and/or objects under different place and time. At this moment, we take 
105 videos from 5 events for testing and the remaining 710 videos for 
training.  For parameters, we set $K=10$ and $N=10$. We learned the representative clusters for each class, and 
then use them to find semantic frames in videos with the same labels. 
Then we evaluation the semantic frames for each video through 
human-factors analysis---the semantic keyframe extraction problem 
demands a human-in-the-loop for evaluation.  We explain our human 
factors experiment in full detail in experiment setup. Our 
ultimate findings demonstrate that our proposed latent FDA with prior 
model is most capable of extraction semantically meaningful keyframes 
among latent FDA and competitive baselines.

\textbf{Video representation.}
For all videos, we extract HOG3D descriptors \cite{Klaser08} every 25 frames (about sampling a frame per second).
To represent videos using local features we apply a bag-of-words 
model, using all detected points and a codebook with 1000 elements.

\textbf{Benchmark methods.}
We make use of SVM as the benchmark method in the experiment. We take the one-vs-all strategy to train a linear MI-SVM classifier using $SVM^{light}$ \cite{Joachims09}, which is very fast in linear time, for each kind of event. Then we choose 10 frames for each video which are far from the margin and close to the margin on positive side. For the frames chosen farthest away from the margin, we refer it SVM(1), while for frames closest to the margin we refer it SVM(2). We also randomly select 10 frames from each video, and we refer it RAND in our experiments. 

\textbf{Experiment setup}
Ten highly motivated graduate students (range from 22 to 30 years' 
old) served as subjects in all of the following human-in-the-loop 
experiments. Each novel subject to the  annotation-task paradigm 
underwent a training process. Two of the authors gave a detailed 
description about the dataset and problem, including its background, 
definition and its purpose. In order to indicate what representative 
and discriminative means for each event, the two authors showed videos 
for each kind of event to the subjects, and make sure all subjects 
understand what semantic keyframes are.
The training procedure was terminated after the subject's 
performance had stabilized. 
We take a pairwise ranking strategy for our evaluation. We extract 10 
frames per video for 5 different methods (SVM(1), SVM(2), LFDA, 
JLFDA and RAND) respectively. For each video, we had about 1000 image 
pairs for comparison. We had developed an interface 
using Matlab to display two image pair and three options (Yes, No and 
Equal) to compare an image pair each time. The students are taught how 
to use the software; a trial requires them to give a ranking: If the left is 
better than the right, then choose 'Yes'; if the right is better than 
the left, choose 'No'. If the two image pair are same, then choose 
'Equal'. The subjects are again informed that better means a better 
semantic keyframe.  The ten subjects each installed the software to their 
computers, and conducted the image pair comparison independently. In 
order to speed up the annotation process, the interface can randomly 
sample 200 pairs from the total 1000 image pairs for each video, and 
we also ask subjects to random choose 10 videos from the test dataset.

\textbf{Experimental Results}
We have scores for each image pair. By sampling 10 videos from each 
event, we at last had annotations of 104 videos. It means our sampling 
videos got from 10 subjects almost cover all test data (105 videos).  
Table (\ref{tab:tab11}) shows the win-loss matrix between five methods by 
counting the pairwise comparison results on all 5 events. It shows 
that JLFDA and LFDA always beat the three baseline methods.  
Furthermore, JLFDA is better than LFDA because it considers a prior 
distribution from training data, which will help JLFDA to find more 
representative frames. See Fig. (\ref{fig:Fig10}) for keyframes extracted with JLFDA. 

\begin{minipage}{\textwidth}
  \begin{minipage}[b]{0.49\textwidth}
    \centering
    \includegraphics[trim = 35mm 88mm 30mm 85mm,clip=true, width=7.0cm]{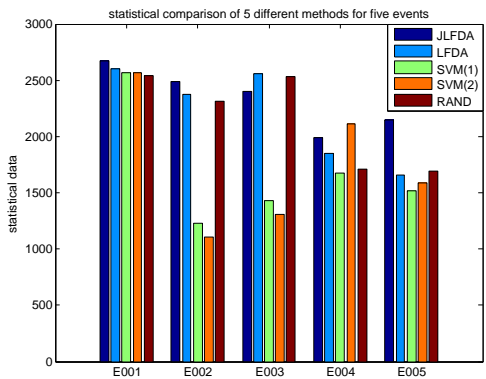}
    \captionof{figure}{Comparison of 5 methods for five events. Higher value, better performance.}
    \label{fig:Fig9}
  \end{minipage}
  \hfill 
  \begin{minipage}[b]{0.49\textwidth}
    \centering
\renewcommand{\arraystretch}{1.8}
\resizebox{7cm}{!} {
\begin{tabular}{|p{0.8cm} | p{0.8cm} p{0.8cm} p{0.8cm} p{1cm} p{1cm} |}
\hline
\multicolumn{1}{|l|}{\multirow{2}{*}{Method}}& \multicolumn{5}{c|}{Win-Loss matrix} \\
\cline{2-6}
& JLFDA & FLDA & RAND & SVM(1) & SVM(2)\\
\cline{1-6}
JLFDA & - & 3413 &  2274 & 2257 & 3758\\
LFDA & 2957 & - & 2309 & 2230 & 3554\\
RAND & 2111 & 2175 & - & 1861 & 2274\\
SVM(1) & 2088 & 2270 & 2010 & - & 2314\\
SVM(2) & 3232 & 3316 & 2113 & 2125 & -\\
\hline
\end{tabular}
}
\captionof{table}{Win-Loss matrix for five methods. It represents how many times methods in each row win methods in column.}
\label{tab:tab11}
    \end{minipage}
  \end{minipage}

We compared the five methods on the basis of Condorcet voting method. We treat 'Yes', 'No' and 'Equal' are voters for each method in the image pairwise comparison process. If 'Yes', we cast one ballot to the left method; else if 'No', we add a ballot to the right method; else do nothing to the two methods. Fig. (\ref{fig:Fig9}) shows ballots for each method on each event. It demonstrates our method JLFDA always beat other methods, except for E004 dataset.
We also compared the five methods based on Elo rating system. For each video, we ranked the five methods according to Elo ranking system. Then, we counted the number of No.1 methods in each event. The results in Table (\ref{Tab2}) show that our method is better than others, except E004. Such results based on Elo ranking is consistent with Condorcet ranking method in Fig. (\ref{fig:Fig9}). E004 is the wedding ceremony event and our method is consistently 
outperformed by the SVM baseline method.  We believe this is due to 
the distinct nature of the E004 videos in which the video scene 
context itself distinguishes it from the other four events (the 
wedding ceremonies typically have very many people and are inside).  
Hence the discriminative component of the methods are taking over, and 
the SVM is able to outperform the Fisher discriminant.  This effect 
seems more likely due to the nature of the five events in the data set 
than the proposed method intrinsically.

\begin{table}[ht!]
\begin{center}
\begin{tabular}{lcccccc}
\hline
\multicolumn{1}{l}{\multirow{2}{*}{Method}}& \multicolumn{5}{c}{the number of No.1 method in each event} \\
\cline{2-6}
& E001 & E002 & E003 & E004 & E005\\
\cline{1-6}
JLFDA & 6 & 7 &  7 & 3 & 7\\
LFDA & 6 & 4 & 4 & 5 & 1\\
SVM(1) & 4 & 4 & 4 & 2 & 4\\
SVM(2) & 6 & 3 & 1 & 7 & 6\\
RAND & 2 & 2 & 4 & 3 & 2\\
\hline
\end{tabular}
\caption{For each video, we ranked the five methods according to Elo 
ranking system. Then, we counted the number of No.1 method one video 
level in each event. For example, E002 has total 20 videos, and JLFDA 
has rank first on 7 videos, while RAND has rank first on only two 
videos. Higher value, better results. It demonstrates that our method 
is more capable at extracting semantically meaningful keyframes.}
\label{Tab2}
\vspace{-7mm}
\end{center}
\end{table}

\begin{figure*}[ht]
\centering
\begin{tabular}{c}
\includegraphics[trim = 10mm 110mm 10mm 110mm,clip=true, width=13.1cm]{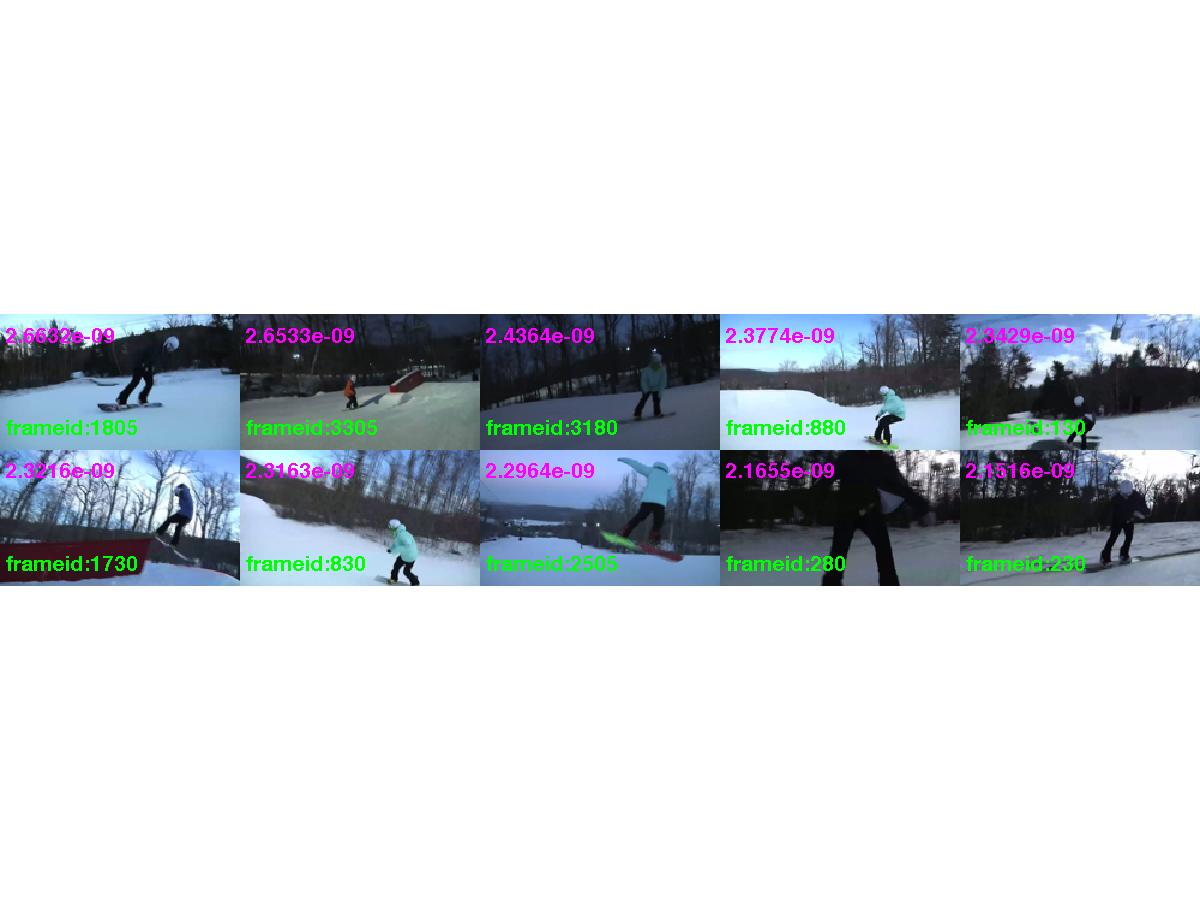}\\
\includegraphics[trim = 10mm 110mm 20mm 110mm,clip=true, width=13.1cm]{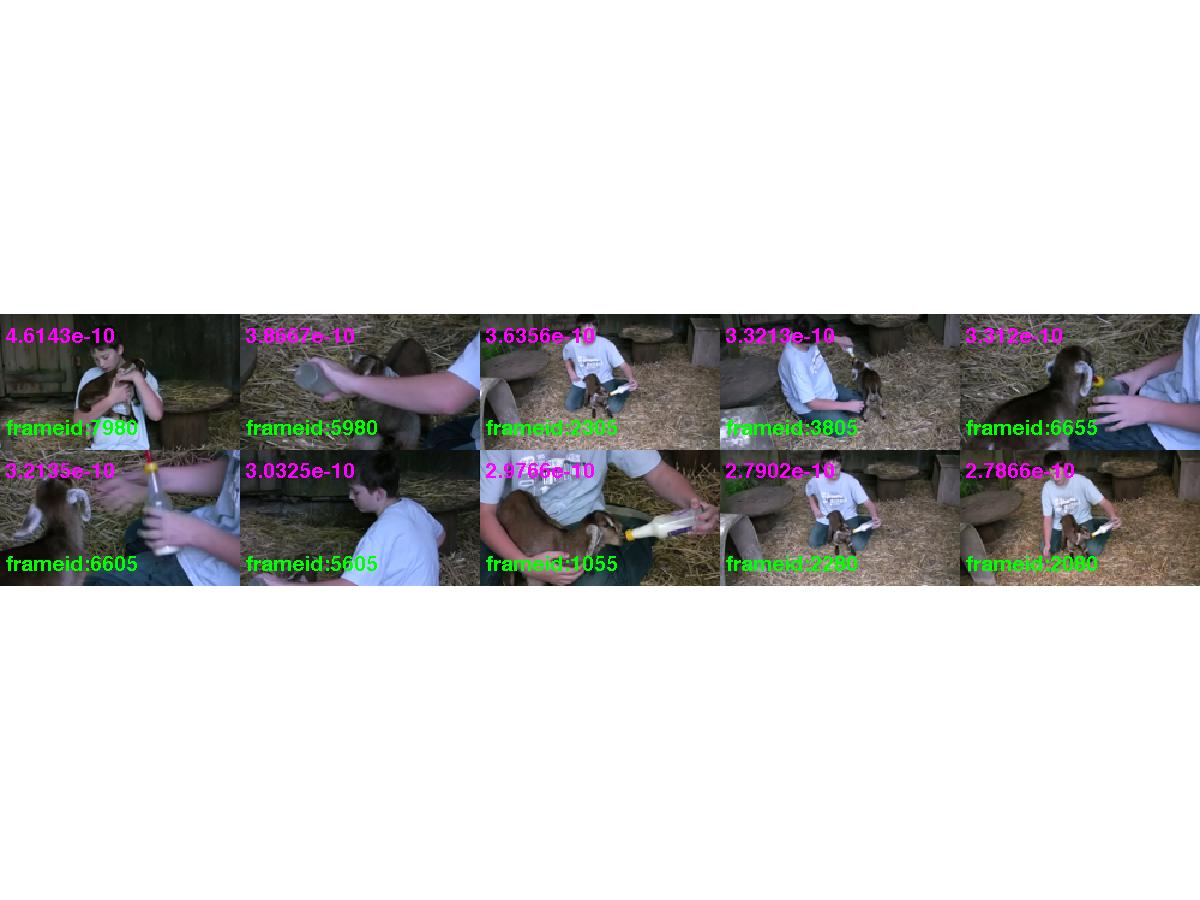}\\
\includegraphics[trim = 10mm 110mm 20mm 110mm,clip=true, width=13.1cm]{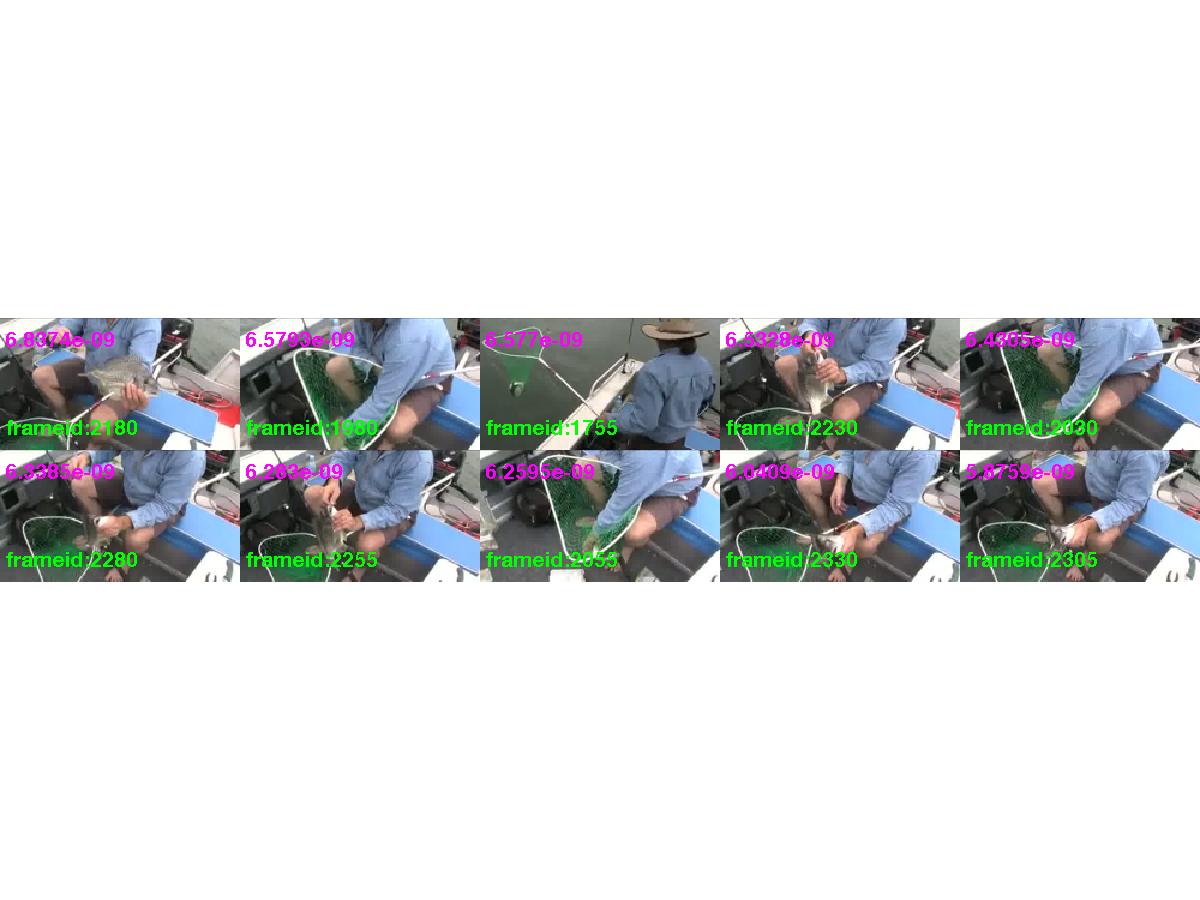}\\
\includegraphics[trim = 10mm 110mm 20mm 100mm,clip=true, width=13.1cm]{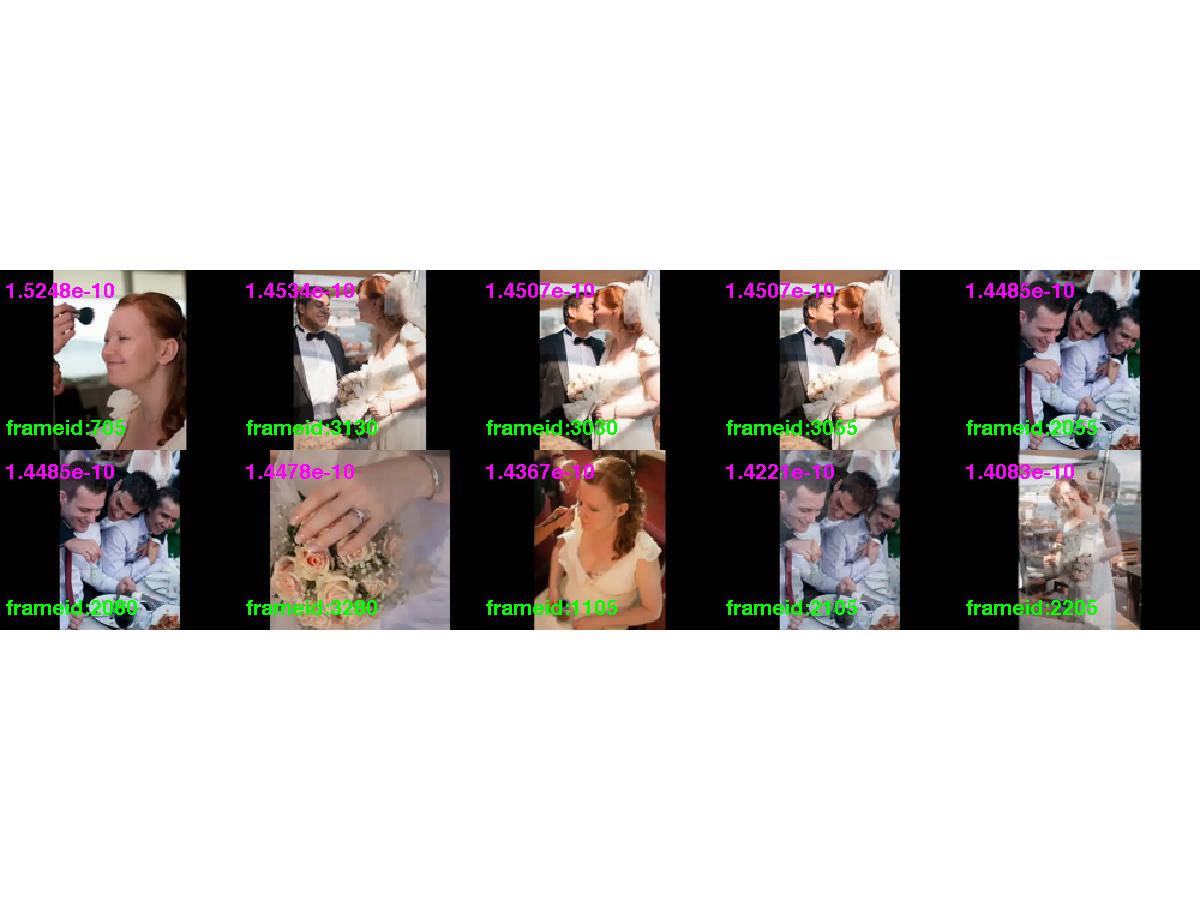}\\
\includegraphics[trim = 10mm 110mm 20mm 100mm,clip=true, width=13.1cm]{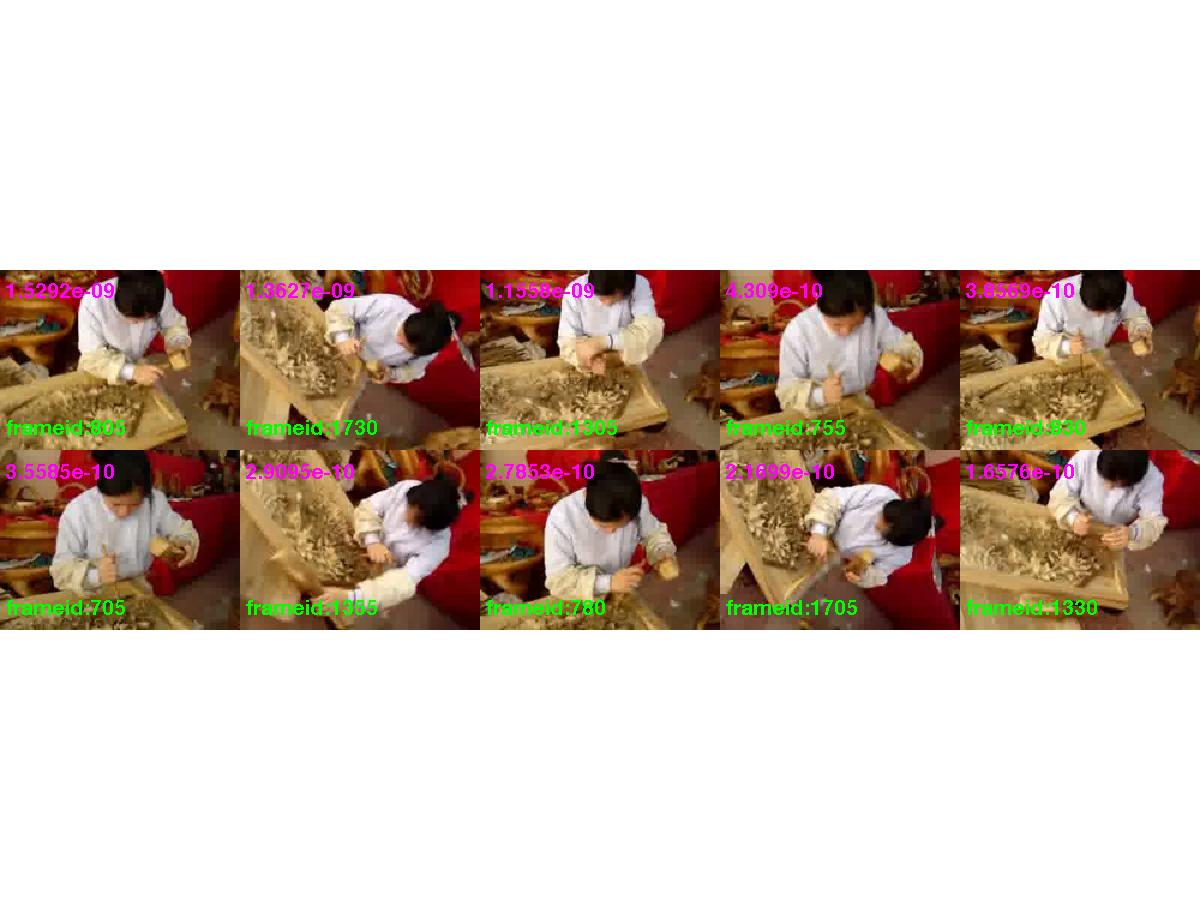}
\end{tabular}
\caption{Sample keyframes from the first five events (each row from top to down): (a) snowboard trick, (b) feeding animal, (c) fishing, (d) marriage ceremony, and (e) wood making. Each row indicates sample results from the same videos for each event. It shows that our method can extract the representative and discriminative images for each kind of events. In other words, we can decide what's happen when we scan the images.}
\label{fig:Fig10}
\end{figure*}

\section{Conclusion}
In this paper, we have presented a latent Fisher discriminant analysis 
model, which combines the latent variable inference and dimension reduction in an unified framework. Ongoing work will extend the kernel trick into the model. We test our method on classification and semantic keyframe extraction problem, and yield quite competitive results. To the best 
of our knowledge, this is the first paper to study the extraction of 
semantically representative and discriminative keyframes---most 
keyframe extraction and video summarization focus on representation 
summaries rather than jointly representative and discriminative ones.  
We have conducted a thorough ranking-based human factors experiment 
for the semantic keyframe extraction on the challenging TRECVID MED11 
data set and found that our proposed methods are able to consistently 
outperform competitive baselines.


{\small
\bibliographystyle{ieee}
\bibliography{videoegbib}
}

\end{document}